\documentclass[preprint]{elsarticle}
\usepackage{amsmath, amsthm, bbm, amsfonts, amssymb, cite, graphicx, caption, enumerate, stmaryrd, ctable, subfigure, cancel}
\usepackage{xypic}\xyoption{all}
\newcommand{\eod}{{${}$\\}}
\newcommand{\cO}{{\mathcal O}}
\newcommand{\x}{{\mathbf x}}
\newcommand{\y}{{\mathbf y}}
\newcommand{\V}{{\mathcal V}}
\newcommand{\wt}{\mathbf{w}}
\newcommand{\bE}{{\mathbb E}}
\newcommand{\bbO}{{\mathbbm 1}}
\newcommand{\bR}{{\mathbb R}}
\newcommand{\cS}{{\mathcal S}}
\newcommand{\cT}{{\mathcal T}}
\newcommand{\cP}{{\mathcal P}}
\newcommand{\cHh}{{\mathcal H}}

\DeclareMathOperator*{\argmax}{argmax}

\newtheorem{thm}{Theorem}
\newtheorem{prop}[thm]{Proposition}
\newtheorem{cor}[thm]{Corollary}
\newtheorem{lem}[thm]{Lemma}
\newtheorem{defn}{Definition}
\newtheorem{assn}{Assumption}

\newtheorem{rem}{Remark}
\newtheorem{eg}{Example}

\begin{document} 

\title{Cortical prediction markets} 
\tnotetext[t1]{Accepted for publication at AAMAS 2014.}
\author{David Balduzzi}\ead{david.balduzzi@inf.ethz.ch}
\address{ETH Z\"urich, Switzerland}

\begin{abstract} 
We investigate cortical learning from the perspective of mechanism design. First, we show that discretizing standard models of neurons and synaptic plasticity leads to rational agents maximizing simple scoring rules. Second, our main result is that the scoring rules are \emph{proper}, implying that neurons faithfully encode expected utilities in their synaptic weights and encode high-scoring outcomes in their spikes. Third, with this foundation in hand, we propose a biologically plausible mechanism whereby neurons backpropagate incentives which allows them to optimize their usefulness to the rest of cortex. Finally, experiments show that networks that backpropagate incentives can learn simple tasks. 
\end{abstract}

\begin{keyword}incentives for cooperation, multiagent learning, biologically-inspired approaches, prediction markets
\end{keyword}

\maketitle 

\section{Introduction} 

How does the brain encode information about the environment into its structure \citep{stanley:13}? Inspired by recent work in prediction markets, this paper investigates cortical learning and the neural code from the perspective of mechanism design \citep{hanson:07, lambert:08, abernethy:11a, abernethy:12, abernethy:13}. To the best of our knowledge it is the first paper to do so.

We start in \S\ref{s:minimal} by modeling neurons as rational agents: that is, agents whose sole aim is to maximize the expected value of an objective function. To do so, we draw on a recent paper showing that \emph{discretizing} standard models of neuronal dynamics \citep{gerstner:02} and learning \citep{song:00} yields a threshold neuron with an online update rule that optimizes a simple objective \citep{bb:12}. By maximizing their objective function, neurons seek to optimally trade off rewards, depending on neuromodulatory signals such as dopamine, with costs, depending on resources expended on synaptic connections \citep{bob:13, bt:13}. 

However, it is not enough that neurons optimize locally. They should collectively converge on useful outcomes. The problem of how a global (cortical) optimization procedure can be implemented at a local (neuronal) level remains open. 

To tackle the problem we turn to mechanism design: How to incentivize populations of rational agents to produce desirable outcomes? 

An inspiring successful application of mechanism design is prediction markets, which aggregate the behavior of self-interested traders into accurate predictions of diverse real-world events \citep{berg:01, ledyard:09}. This has motivated research on payment schemes that encourage agents to trade in markets if the price distribution differs from their beliefs \citep{hanson:07}. Of particular interest are \emph{proper} scoring rules: payment schemes that incentivize rational agents to truthfully report their beliefs \citep{lambert:08}. 

Our next step, \S\ref{s:design}, is therefore to analyze neuronal objective functions \emph{as payment schemes}. This has implications in two directions. First, since the neuronal objective function decomposes as a sum over synapses, we model synapses as rational agents trading in a neuronal market, \S\ref{s:npm}. Second, we model neurons as rational agents trading in a cortical prediction market, \S\ref{s:cpm}.

Our main result, Theorem~\ref{t:proper}, establishes a striking connection between prediction markets and cortical learning: neuronal objective functions are proper scoring rules. The remainder of the paper applies two corollaries of Theorem~\ref{t:proper} to show that well-functioning neuronal markets form a foundation for a well-functioning cortical market -- thereby gluing together the two perspectives. 

Corollary~\ref{t:stdp_proper} shows that synaptic weights encode the utility expected after pre- and post- synaptic spikes. This partially answers the question posed earlier: ``How does the brain encode information about the environment into its structure?''

More importantly, the corollary provides a foundation for cooperative learning. Consider the following basic schema to incentivize rational agents to collaborate:

(i) each agent estimates its usefulness to other agents, 

(ii) incorporates the estimate into its reward function and 

(iii) thus maximizes its usefulness to the collective.\\
To implement the schema, neurons must estimate their usefulness. Corollary~\ref{t:stdp_proper} implies that synaptic weight $\wt_{ij}$ quantifies how useful spikes from $n_i$ are to $n_j$, when $n_j$ spikes. More generally, the sum of outgoing synaptic connections quantifies how useful a neuron's outputs are to the rest of the system. We therefore define the \emph{usefulness} of a neuron as, roughly, the sum of its downstream weights, \S\ref{s:bp}.

In line with the schema we then show, Corollary~\ref{t:feedback}, that incorporating feedback into reward functions causes neurons to (i) estimate their usefulness and (ii) maximize the estimate. This provides a new interpretation of a spike-based backpropagation scheme \citep{roelfsema:05} that is closely related to error-backpropagation \citep{rumelhart:86}. 

In short, well-functioning neuronal markets, with synapses faithfully reporting expected utilities, can be used to build well-functioning cortical markets. 

Finally, experiments in \S\ref{s:experiments} confirm our theoretical results.

\paragraph{Scope and related work}
A well-studied framework in neuroscience is based on the idea that neurons infer the probabilities of external events, which are encoded into probabilistic population codes, see e.g. \citep{Boerlin:2011ys}. By contrast, we emphasize decisions over inferences.  We are concerned with how neurons act, rather than what they infer. The two perspectives are related and it may turn out, as in prediction markets where prices can encode probabilities, that the population coding and mechanism design approaches lead to the same destination.

Note that our goal is to show methods from mechanism design can be fruitfully applied to fundamental questions in neuroscience. We do not advocate \emph{specifically} for the scoring rules described below. These were derived from standard, but simple, neurophysiological models. With additional work it should be possible to extend our results to more realistic models. 

This work is inspired by a striking connections that has recently been discovered between market scoring rules and no-regret learning \citep{chen:10}, and related work suggesting that carefully designed markets could be used to aggregate hypotheses generated by populations of learning algorithms \citep{lay:10, storkey:11, abernethy:11a}.

\section{A minimal model}
\label{s:minimal}

At first glance, the models developed by neuroscientists are quite different from the rational agents studied in game theory. To build a bridge we utilize recent work \emph{discretizing} a standard model from the neuroscience literature \citep{bb:12}.

\subsection{Discretized neurons}
Consider a system of $N$ binary neurons $\{n_j\}_{j=1}^N$. Let $\cO = \{0,1\}^N$ denote the set of possible states. Each neuron is connected to a subset of the system. Suppose neuron $n_j$ has $K_j\ll N$ synapses. We model the restriction of the total system state to the subset received by neuron $n_j$ with a mask projecting from $\{0,1\}^N$ to $\{0,1\}^{K_j}$
\begin{equation}
	\label{e:mask}
	\varphi_j:\cO\rightarrow \{0,1\}^{K_j}:\x=(x_1,\ldots, x_N)\mapsto (x_i)_{\{i|i\rightarrow j\}}.
\end{equation}
Neuron $n_j$ is equipped with a $K_j$-vector of synaptic weights, $\wt_j\in \cHh_j=\bR^{K_j}$. Given input $\x\in\cO$, the neuron outputs a 0 or 1 according to 
\begin{equation}
	\label{e:threshold}
	f_{\wt_j}(\x) := \begin{cases}
		1 & \text{if }\langle\wt_j,\varphi_j(\x)\rangle-\vartheta>0\\
		0 & \text{else}
	\end{cases}
\end{equation}
for some fixed $\vartheta$ constant across all neurons.

To simplify the exposition, we drop $\varphi_j$ from the notation and let $\cHh:=\bR^N$ denote the space of synaptic weights -- where synapses that do not physically exist are implicitly clamped to zero. Thus, we treat entire system states as inputs to a neuron -- when in fact the mask projects out most inputs.

\begin{defn}
	Suppose we have utility function $\mu:\cO\rightarrow \bR$. Following \citep{bb:12}, define \textbf{reward function}
\begin{equation}
	\label{e:reward}
	R(\x,\wt_j,\mu_j) 
	= \underbrace{\mu_j(\x)}_{\text{utility}}\cdot \underbrace{(\langle\wt_j,\x\rangle-\vartheta)}_{\text{margin}}\cdot \underbrace{f_{\wt_j}(\x)}_{\text{selectivity}}
\end{equation}
\end{defn}
Examples of utility functions are provided in \S\ref{s:utility} and \S\ref{s:feedback}.

\begin{rem}[notation for spikes]
	Note that $f_{\wt_j}(\x)$, $\x_j$, $\bbO_{\wt_j}$,  and $\bbO_j$ all denote the output of neuron $n_j$; emphasizing the function producing the output, that the output is also an input (one of many forming a vector) to other neurons, or the indicator-function aspect of the output respectively. 
	We use  $\bbO_{ij} := \bbO_i\cdot \bbO_j$ to indicate the cospiking of neurons $n_i$ and $n_j$.
\end{rem}

Ignoring costs for a moment, suppose neurons maximize $\bE_{(\x,\mu)\sim P}\Big[R(\x,\wt,\mu)\Big]$, where $P(\x,\mu)$ is the joint distribution on spiking inputs and neuromodulators. 

The reward function is continuously differentiable (in fact, linear) as a function of $\wt$ everywhere except at the kink $\langle\wt,\x\rangle=\vartheta$ where it is continuous but not differentiable. We can therefore perform gradient ascent to obtain synaptic updates
\begin{equation}
	\label{e:stdp}
	\Delta \wt_{ij} \propto \mu_j(\x)\cdot \x_i \cdot f_{\wt_j}(\x) = \mu_j(\x)\cdot \bbO_{ij}.
\end{equation}
In short, if $n_j$ receives input $\bbO_i$ and subsequently spikes $\bbO_j$, then synapse $i\rightarrow j$ is modified proportionally to $\mu_j(\x)$. The main theorem in \citep{bb:12} derives the above equations by discretizing standard models of neuronal dynamics and learning: 

\begin{thm}[discretized neurons, \citep{bb:12}]\label{t:limit}
 	The fast time constant limit of Gerstner's Spike Response Model \citep{gerstner:02} is \eqref{e:threshold}. Taking the fast time constant limit of STDP \citep{song:00} yields \eqref{e:stdp} with $\mu_j(\x)=1$. Finally, STDP is gradient ascent on a reward function whose limit is $\eqref{e:reward}$.
\end{thm}

Spike-timing dependent plasticity is prone to overpotentiation \citep{song:00}, leading to epileptic seizures. In the neuroscience literature, weights are typically controlled with a depotentation bias. We take an alternative approach, by introducing a \emph{regularizer} $A_\bullet(\wt)$, which quantifies the \emph{resource costs} incurred by high synaptic weights \citep{Hasenstaub:2010fk, bb:12, bob:13}.

The optimal weights are then computed according to 
\begin{align}
	\label{e:regmax}
	\wt^*_j & := \argmax_{\wt\in\cHh}\, \bE_{P}\Big[\cS_\bullet(\x;\wt)\Big]\\
	& := \argmax_{\wt\in\cHh}\, \bE_{P}\Big[R(\x,\wt,\mu) - A_\bullet(\wt)\Big]
\end{align}
where \emph{\textbf{scoring rule}} $\cS_\bullet(\x,\wt)$ balances rewards $R(\x,\wt,\mu)$ against costs $A_\bullet(\wt)$. We consider two standard regularizers taken from machine learning \citep{shalev:07} and a third, more biologically plausible, taken from \citep{bb:12}:
\begin{equation*}
	\begin{cases}
		A_2(\wt_j) = \frac{1}{2\eta} \|\wt_j\|^2_2                                       & \ell_2  \\
		A_H(\wt_j) = \frac{1}{\eta} \sum_{i}\wt_{ij}\log\wt_{ij}                                      & \ell_H  \\
		A_1(\wt_j) = \frac{1}{\eta} \|\wt_j\|_1,\text{ where }0\leq \wt_{ij}\leq1\text{ for all }i.  & \ell_1  
	\end{cases}
\end{equation*}
Clearly $\ell_H$ is not a norm -- we find the notation convenient.

Computing gradient ascent yields online updates
\begin{equation}
	\label{e:updates}
	\Delta \wt_{ij} \propto \mu_j(\x)\cdot \bbO_{ij} - \frac{1}{\eta}\cdot \begin{cases}
		\wt_{ij} & \ell_2 \\
		\log \wt_{ij} + 1 & \ell_H \\
		1 & \ell_1
	\end{cases}
\end{equation}

\begin{rem}[regularizers]
	Each regularizer has points in its favor. The $\ell_1$ regularizer provides a simple interpretation of the saturated synaptic weights observed in some neurophysiological models \citep{fusi:07}. The $\ell_2$ regularizer allows negative synaptic weights, corresponding to inhibitory synapses.  Finally, $\ell_H$ results in weights that can be interpreted as a probability distribution and is closely related to Hanson's logarithmic market scoring rule \citep{chen:07}.	
\end{rem}

\subsection{Utility functions} 
\label{s:utility}

Three biologically inspired utility functions are:

\setcounter{eg}{0}
\renewcommand{\eg}{\emph{Example U\arabic{eg}}. }
\addtocounter{eg}{1}
\begin{eg}(Feedforward, frequency)\textbf{.}\;
	Utility function $\mu(\x)=1$ encourages neurons to spike for inputs that are frequent and contain many spikes. 
\end{eg}

\addtocounter{eg}{1}
\begin{eg}(Feedforward, invariance)\textbf{.}\;
	A more interesting utility function takes inputs over consecutive time steps $\x=(\x^{(t-1)},\x^{(t)})$ as input and sets $\mu(\x)=f_\wt(\x^{(t-1)})$. This encourages neurons to learn \emph{stable} patterns containing many spikes, i.e. those that cause it to spike \emph{twice} consecutively. The utility function can be extended across multiple time steps, possibly with a temporal discount factor.
\end{eg}

\addtocounter{eg}{1}
\begin{eg}(Neuromodulators)\textbf{.}\;
	Neuromodulatory systems signaling global rewards can be modeled via $P(\nu|\x)$ where $\nu$ is a real-valued random variable: positive outcomes are reinforced and conversely. The utility is then $\mu(\x):= \bE_{\nu\sim P(\nu|\x)}[\nu|\x]$, where the expectation is with respect to the distribution on neuromodulators.
\end{eg}

A fourth utility function is discussed in \S\ref{s:feedback}.

\section{Neuronal prediction markets}
\label{s:design}

Scoring rules are schemes for paying agents based on their reports. Proper scoring rules, which incentivize agents to report truthfully, have proven useful in a wide range of settings including weather forecasts \citep{brier:50}, prediction markets \citep{hanson:07, lambert:08} and crowdsourced learning mechanisms \citep{abernethy:11a, abernethy:12}. 

Our main result, Theorem~\ref{t:proper}, is that the scoring rules $\cS_\bullet$ in \eqref{e:regmax} are proper for all three regularizers. The upshot is that a neuron's synaptic weights faithfully encode\footnote{``Truthful reporting'' is not appropriate when referring to neurons. We use the phrase ``faithful encoding'' instead.} expectations about rewards after pre- and post- synaptic spiking activity. The form of the encoding depends on the regularizer.

\subsection{Synapses as rational agents} 
\label{s:npm}

This subsection argues that synapses are analogous to traders, operating within a neuronal market, that attempt to maximize their payout relative to their expenditures.

Prediction market traders buy and sell contingent securities. The simplest case is an \emph{Arrow-Debreu security}, which pays out \$1 if an outcome belongs to a particular set, and \$0 otherwise \citep{arrow:54}. For example an Arrow-Debreu security could pay \$1 if and only if candidate $X$ wins an election. The price a trader will pay depends on her expectations about whether $X$ will win. It turns out that the prices of securities in well-designed, liquid markets reliably aggregate traders' diverse, private information into public estimates of the probabilities of outcomes \citep{hanson:06}.

\begin{center}
\begin{tabular*}{.6\textwidth}{@{\extracolsep{\fill}}| l | l | l |}
	\hline
	$n_j$ & neuron & market \\
	$i\rightarrow j$ & synapse & trader \\
	$\bbO_i$ & spike & security \\
	$A_\bullet(\wt_{ij})$ & regularizer at $i$ & cost to $i\rightarrow j$  \\
	$\wt_{ij}\bbO_i$ & weight $\times$ spike & $\bbO_i$s bought by $i\rightarrow j$ \\
	$\langle\wt_{j},\x\rangle$ & total current & bundle of securities \\
	$\langle\wt_{j},\x\rangle\bbO_j$ & current $\times$ spike & collective bid \\
	$\mu_j(\x)\wt_{ij}\bbO_{ij}$ & reward of $i\rightarrow j$ & payout to $i\rightarrow j$\\
	\hline
\end{tabular*}
\end{center}

Since the neuronal scoring rule decomposes into sum $\cS_\bullet=\sum_i \cS_\bullet^i$, we can model not only neurons, but also synapses, as rational score-maximizing agents. Synapse $i\rightarrow j$ receives payment
\begin{equation}
	\label{e:sscore}
	\cS_\bullet^i := \bE_P\big[(\wt_{ij}\bbO_i-\vartheta)\cdot\mu_j(\x)\bbO_j-A_\bullet(\wt_{ij})\big],
\end{equation}
where $\bbO_j$ depends on vector $\wt_{j}$ and couples the synapses. 

Synapse $i\rightarrow j$ invests amount $A_\bullet(\wt_{ij})$ to set its weight to $\wt_{ij}$. In return, it receives quantity $\wt_{ij}$ of security $\bbO_i$. 

Like paper money, the securities $\bbO_i$ have no intrinsic worth. Instead, they are bundled into total current  $\langle\wt_{j},\x\rangle$. If the bundle exceeds threshold $\vartheta$ then $n_j$ spikes. That is, $n_j$ uses the bundle to bid on an extrinsic event: the utility $\mu_j(\x)$.

After bidding, neuron $n_j$ receives payout $\mu_j(\x)\langle\wt_{j},\x\rangle\bbO_j$, of which it distributes an amount $\mu_j(\x)\wt_{ij}\bbO_{ij}$ to each synapse proportional to its contribution $\wt_{ij}\bbO_i$ to the bundle. Synapses only receive payouts when they spike. Payouts can be positive or negative.

Summarizing, synapses optimize the payout,  $\mu_j(\x)\wt_{ij}\bbO_{ij}$ resulting from their contribution $\wt_{ij}\bbO_i$ to the collective bid, against their cost $A_\bullet(\wt_{ij})$. The neuron's bid $\langle\wt_{j},\x\rangle\bbO_j$ is thus a collective prediction of high utility by its synapses.


\subsection{Proper scoring rules}

The remainder of this section uses \emph{properness} to precisely quantify how synaptic weights relate to utility expectations.

\begin{defn}\label{d:proper}
	Let $\cP_\cO$ be a set of probability distributions on states $\cO$ and define a \textbf{property} as a function $\Gamma:\cP_\cO\rightarrow \cHh$. 
	Scoring rule $S:\cO\times\cHh\rightarrow \bR$ is \textbf{proper} \citep{lambert:08} for property $\Gamma:\cP_\cO\rightarrow \cHh$ if for all $P\in \cP_\cO$
	\begin{equation}
		\label{e:proper-sc}
		\Gamma(P)\in\argmax_{\wt\in\text{range}(\Gamma)}
		\bE_{P}\big[S(\x;\wt)\big].
	\end{equation}
\end{defn}	

Properness is the common-sensical requirement that the true value, $\wt=\Gamma(P)$, is a score maximizer, $\wt\in\argmax \bE_P[S]$. In short: ``you get what you think you are paying for''.

Proper scoring rules can be constructed as follows \citep{abernethy:12}. Given functions $\rho:\cO\rightarrow\cHh$ and $F:\cHh\rightarrow \bR$, define
\begin{equation*}
	S_F:\cO\times\cHh\rightarrow\bR:(\x;\wt) \mapsto - D_F(\rho(\x),\wt) - F(\rho(\x))
\end{equation*}
where $D_F(\x,\y):=F(\x) - F(\y) - \langle \nabla_F(\y),\x-\y\rangle$ is the Bregman divergence. It is shown in \citep{abernethy:12} that:

\begin{prop}[linear proper scoring rules]\label{t:jake}\eod
	If $F$ is convex then $S_F$ is a proper scoring rule for linear property $\Gamma:P\mapsto \bE_{P}[\rho(\x)]$. 
\end{prop}

\subsection{Proper scoring for discretized neurons}

This section adapts Proposition~\ref{t:jake} to discretized neurons. As a warmup, we show that dropping the selectivity term from \eqref{e:reward} yields proper scoring rules.

\begin{lem}\label{t:rproper}
	Let $\cT_\bullet(\x;\wt_j) := \mu_j(\x)\cdot \big(\langle\wt_j,\x\rangle -\vartheta\big) - A_\bullet(\wt_j)$ be scoring rules. These are \emph{proper} for $\Gamma^{\cT}_\bullet:\cP_\cO\rightarrow \cHh=\bR^N$,
	\begin{equation*}
		\label{e:rproper}
		\Gamma^{\cT}_\bullet:P  \mapsto G_\bullet\Big(\bE_{P}\big[\mu_j(\x)\cdot \x\big]\Big)
		\text{ for } \begin{cases}
		G_2({\mathbf v}) = \eta\cdot {\mathbf v}\\
		G_H({\mathbf v}) = e^{\eta\cdot {\mathbf v}-1}\\
		G_1({\mathbf v}) = \bbO_{\eta\cdot {\mathbf v}>1},
		\end{cases}
	\end{equation*}
	where $\bbO_{\eta\cdot\bullet>1}$ an $N$-vector of indicator functions returning $1$ when $\eta\cdot\bullet>1$ and 0 otherwise.	
\end{lem}

\begin{proof}
	We drop $\vartheta$ since it is independent of $\wt_j$. Define hypothesis space $\cHh = \bR^N$ and map
	\begin{equation*}
		\rho_\mu:\cO\rightarrow \cHh:\x\mapsto \mu(\x)\cdot \x.
	\end{equation*}
	We consider the three cases in turn.
	
	Observe that convex function $F_{2}(\x)=\frac{1}{2\eta}\|\x\|^2_2$ yields scoring rule $S_2(\x,\wt_j) = \langle \mu(\x)\cdot\x,\wt_j\rangle - \frac{1}{2\eta}\|\wt_j\|_2^2$, which implies $\cT_2$ is proper by Proposition~\ref{t:jake}.

	For $\cT_H$, restrict $\cHh$ to the subset of $\bR^N$ where $\sum_i \exp(\frac{1}{\eta}\wt_{ij}))=1$ and define $\psi:\cHh\rightarrow\cHh=\bR^N$ taking $\wt_{ij}\mapsto\frac{\exp(\frac{1}{\eta}\wt_{ij})}{\sum_i \exp(\frac{1}{\eta}\wt_{ij}))} = \exp(\frac{1}{\eta}\wt_{ij})$. Convex function $F_H(\x)=\eta\log(\sum_{i=1}^n\exp(\frac{1}{\eta} \x_i))$ yields
	\begin{align*}
		S_H(\x,\wt_j) 
		& = F_H(\wt_j) + \left\langle \psi(\wt_j),\mu_j(\x)\cdot\x-\wt_j\right\rangle\\
		& = \big\langle \psi(\wt_j),\mu_j(\x)\cdot\x \big\rangle - \frac{1}{\eta} \big\langle\psi(\wt_j),\log\psi(\wt_j)\big\rangle,
	\end{align*}
	since $F_H(\wt_j) = 0$. By Proposition~\ref{t:jake} it follows that $\cT_H$ is proper for linear property $\bE_P[\eta\cdot {\mathbf v}]$. The result follows for $\Gamma^{\cT}_H$ since $e^{\bullet-1}$ is monotonic. We use the nonlinear ``$\exp$'' representation since it directly corresponds to synaptic weights which will be useful in Theorem~\ref{t:proper}.

	Proposition~\ref{t:jake} does not apply to $\cT_1$, so we derive properness by other means. Computing gradients obtains
	\begin{equation*}
		\Delta \wt_j \propto \bE_{P} \big[\mu_j(\x)\cdot \x\big] - \frac{1}{\eta}
	\end{equation*}
	which has a stationary point when all synaptic weights are equal to the scalar $\eta$. The stationary point is unstable -- a local minima rather than maxima. Synapses with $\bE_{P}[\eta\cdot\mu_j(\x)\cdot \bbO_i]>1$ are forced to boundary condition $\wt_{ij}=1$; others are forced to 0 (for simplicity we assume no expectation is precisely $1$). 

	The range of $\Gamma^{\cT}_1$ is the set of $N$-vectors of 0s and 1s. Any $\wt\in \text{\emph{range}}(\Gamma^{\cT}_1)$ differing from $\Gamma^{\cT}_1(P)$ has non-zero gradient and hence a lower score, implying $\Gamma^{\cT}_1$ is proper.
\end{proof}

The selectivity term in \eqref{e:reward} introduces a complication into the scoring rule: potentiating a synaptic weight may cause a neuron to stumble over a sharp change in its utility function that is hidden by the selectivity term. Although the reward function is continuous in $\wt$ its derivative is not: there is a kink. We bound the jump after crossing a kink via
\begin{assn}[no nasty surprises]\label{a:separable}	
	If \\$\Delta_{ij}=\bE_P\big[\mu_j(\x)\bbO_i\bbO_{\wt_j}\big] -\partial_{i}A_\bullet(\wt_j)> 0$ then there exists $\epsilon>0$ such that
	\begin{equation*}
		\bE_P\big[\mu_j(\x)\bbO_{\epsilon\cdot\Delta_{ij}}\big]> - \Delta_{ij},
	\end{equation*}	
	where $\bbO_{\Delta_{ij}}:=\bbO_{\wt_j+\epsilon\cdot\Delta_{ij}}-\bbO_{\wt_j}$.
\end{assn}

Assumption~\ref{a:separable} implies that sufficiently small synaptic updates, Eq.~\eqref{e:updates}, always increase a neuron's score:

\begin{lem}[smooth ascent]
	\label{t:separable}
	Under Assumption~\ref{a:separable}, if $\Delta_{ij}>0$ then there exists $\epsilon>0$ such that
	\begin{equation*}
		\bE_P\big[\cS_\bullet(\x,\wt_j + \epsilon\cdot\Delta_{ij})\big] > \bE_P\big[\cS_\bullet(\x,\wt_j)\big]
	\end{equation*}
	and similarly for $\Delta_{ij}<0$.
\end{lem}

\begin{proof}
	Straightforward computation.
\end{proof}

Informally, if high utility follows $n_i$ and $n_j$ cospiking, then Assumption~\ref{a:separable} says that the utility of \emph{new} inputs, causing $n_j$ to spike when synapse $i$ increases by $\Delta_{ij}$, is not \emph{too} negative. If the assumption fails then the neuron will continuously potentiate and depotentiate synapse $i$ as the gradient jumps from positive to negative. This is analogous to the behavior of a perceptron when confronted with classes that are not linearly separable.
 
Nasty surprises can be avoided in at least two ways. First, by designing the utility function so that it behaves well with respect to the distribution the neuron encounters. Second, by allowing neuron $n_j$ to modify its regularization parameter $\eta_j$. Going further, one could introduce additional degrees of freedom by associating an $\eta_{ij}$ with each synapse (note the regularizers are sums over synapses) that is tweaked when a neuron detects that one of its synapses jumps back and forth. We do not pursue these ideas here.

Before proving our main result, we introduce some notation. Given $\wt$, let $\bbO_\wt:=f_\wt^{-1}(1)\subset \cO$ and let ${\mathbbm 2}^\cO$ denote the powerset of $\cO$. Enlarge the hypothesis space to $\cHh':= {\mathbbm 2}^\cO \times \cHh$ with embedding $\psi:\cHh\rightarrow \cHh':\wt\mapsto(\bbO_\wt, \wt)$. Let $\wt^*= \argmax_{\wt\in\cHh} \bE_P[\cS_\bullet]$.

\begin{thm}[neuronal scoring rules are proper]\label{t:proper}	
	Under Assumption~\ref{a:separable}, scoring rules $\cS_\bullet$ are \emph{proper} for property $\Gamma_\bullet:{\mathcal P}\rightarrow\cHh'={\mathbbm 2}^\cO \times \cHh$,
	\begin{equation}
		\label{e:proper}
		\Gamma_\bullet:
		P \mapsto \Big(\bbO_{\wt^*}, G_\bullet\big(\bE_P\big[\mu(\x)\cdot \x \cdot\bbO_{\wt^*}\big]\big)\Big),
	\end{equation}
	where $G_\bullet({\mathbf v}) \in\{\eta\cdot {\mathbf v}, e^{\eta\cdot{\mathbf v}-1}, \bbO_{\eta\cdot{\mathbf v}>1}\}$ depending on the choice of regularizer.
\end{thm}

\begin{proof}
	Property $P\mapsto \bbO_{\wt^*}$ is proper by construction; we therefore focus on the synaptic term $G_\bullet(\cdot)$ in Eq.~\eqref{e:proper}.
		
	Computing gradients for $\cS_2$ and $\cS_H$ yields stationary points
	\begin{gather*}
		\wt^* = \bE_{P}\big[\eta\cdot \mu(\x)\cdot\x\cdot \bbO_{\wt^*}\big]
		\quad\text{ and }\\
		\wt^* \propto \exp\Big(\bE_{P}\big[\eta\cdot\mu(\x)\cdot \x\cdot \bbO_{\wt^*}\big]-1\Big)
	\end{gather*}
	respectively which are stable maxima under Assumption~\ref{a:separable} by Lemma~\ref{t:separable}. As argued before, a weight vector $\wt$ that does not have zero-gradient cannot be a maxima, and the argument follows from Lemma~\ref{t:rproper}. 
	
	Similar reasoning applies to $\cS_1$.
\end{proof}

\begin{rem}[indirect elicitation]
	\label{e:eliciting}
	Eliciting properties from distributions was studied in \citep{lambert:08}, which drew a distinction between elicitable and \emph{directly} elicitable properties. For example, the variance can only be elicited by a scoring rule if the mean is elicited as well. Similarly, $G_\bullet(\bE_P[\mu(\x)\cdot\bbO_{i}\bbO_{\wt^*}])$ cannot be elicited directly, but only in conjunction with $\bbO_{\wt^*}$.
\end{rem}

Neurons only modify their synapses to incorporate rewards when spiking, Eq.~\eqref{e:stdp}. This encourages specialization, but also implies that individual neurons may never discover that spiking for certain inputs results in very high utility.
More formally, the kink makes $\cS_\bullet$ non-convex, so gradient ascent is not guaranteed to find the global optimum.

Nevertheless, the relationship between synaptic weights and expected utilities in Theorem~\ref{t:proper} still holds:

\begin{cor}[synaptic code]\label{t:stdp_proper}
	Let $\tilde{\wt}$ be the (in general local) maximum of $\cS_\bullet$ obtained by gradient ascent with Eq.~\eqref{e:updates}.
	If Assumption~\ref{a:separable} holds then $\tilde{\wt}$ satisfies
	\begin{equation}
		\label{e:stdp_proper}
		\tilde{\wt} = G_\bullet\Big(\bE_P\big[\mu(\x)\cdot \x \cdot\bbO_{\tilde{\wt}}\big]\Big).
	\end{equation}
\end{cor}

Note that Eq.~\eqref{e:stdp_proper} is not in closed form since $\tilde{\wt}$ appears on both the left- and right-hand sides. 

\begin{proof}
	Since local maxima are stationary points, the proof follows the same argument as Theorem~\ref{t:proper}.
\end{proof}

A discretized neuron $n_j$ thus faithfully encodes two properties of its input distribution. First, its spikes encode a set of inputs for which spiking is locally optimal. Second, its synaptic weights encode the \emph{expected utility per synapse} when $n_i$ and $n_j$ co-spike. 

\begin{rem}[neural code]
	Corollary~\ref{t:stdp_proper} provides an interesting interpretation of the meaning of spikes. A neuron spikes if the dot product $\langle\wt_{j},\x\rangle$ is above threshold $\vartheta$. That is, neuron $n_j$'s spike means that the \emph{current} system state $\x$ is significant (above threshold) when evaluated against the utility expectations $\wt_{j}$ that were \emph{previously encoded} into $n_j$'s structure.	
\end{rem}

Similarly to how stock price movements encode information about which sectors of an economy are expected to yield high profits in the near future; spikes and synaptic weights encode expectations about future rewards.

\section{Cortical prediction markets}
\label{s:bp}

This section investigates how neurons can estimate their usefulness to downstream neurons, and so allocate their resources such that the benefit to \emph{other} neurons is maximized. In short, we introduce a utility function that incentivizes neurons to optimize their usefulness to other neurons.

\subsection{Backpropagation: errors or incentives?} 
\label{s:feedback}

To provide context, we recall related work on incorporating spikes into a reward signal. Neuromodulators provide a primary reward system. However, neurons whose actions do not directly result in pleasure or pain may require more indirect incentives. In machine learning, multilayer networks are often trained by backpropagating errors \citep{rumelhart:86}. However, backpropagation (BP) is biologically implausible -- it requires pathways for backpropagating errors which have not been observed in cortex \citep{roelfsema:05}. 

As an alternative, \citep{roelfsema:05} proposed attention-gated reinforcement learning (AGREL), which uses feedback spikes as attention signals to modulate learning. AGREL abstracts two features of feedback (NMDA) connections in cortex: (i) they prolong, but do not initiate, spiking activity and (ii) they have a multiplicative effect on synaptic updates.

\vspace{-4mm}
\begin{equation*}
	\xymatrix{
	*+[o][F]{n_i}\ar@{->}[r]^{\wt^{ff}_{ij}} & *+[o][F]{n_j} \ar@/_/[r]_{\wt_{jk}} & *+[o][F]{n_k}\ar@/_/[l]_{\wt^{fb}_{kj}}
	}
\end{equation*}

AGREL updates feedforward weights according to 
\begin{equation}
	\label{e:agrel}
	\Delta \wt_{ij} \propto  \wt^{fb}_{kj} \x_k \cdot \x_i \x_j\cdot(1-\x_j)\cdot f(\delta),
\end{equation}
where $f(\delta)$ is a global reward signal. Here, neurons have real-valued outputs and $(1-\x_j)$ is a regularizer that prevents $n_j$ from overactivating. The main result of \citep{roelfsema:05} is that average weight changes $\bE\big[\Delta \wt_{ij}\big]$ under \eqref{e:agrel} coincide with BP. AGREL thus provides a biologically plausible substitute for BP.

Inspired by AGREL, we introduce a $4^{th}$ utility function:

\addtocounter{eg}{1}
\begin{eg}(Feedback)\textbf{.}\;
	Identify disjoint upstream and downstream populations, $\x^{ff}$ and $\x^{fb}$ respectively, and define $\cHh^{ff}$ and $\cHh^{fb}$ by clamping weights not in the respective populations to zero using Eq~\eqref{e:mask}. Define the feedback utility as $\mu^{fb}_j(\x) := \langle \wt^{fb}_j, \x\rangle$ for $\wt_j^{fb}\in \cHh^{fb}$.
\end{eg}

A neuron with feedback utility maximizes
\begin{equation}
	\label{e:feedback_max}
	\bE_{P}\Big[\langle\wt^{fb}_j,\x\rangle\big(\langle\wt^{ff}_j,\x\rangle-\vartheta\big)\bbO_j - A_\bullet(\wt_j)\Big]
\end{equation}
and so aligns its feedforward $\langle\wt^{ff}_j,\x\rangle$ and feedback $\langle\wt^{fb}_j,\x\rangle$ current whenever the neuron itself spikes. 

Computing gradient ascent on scoring rule 
$\cS^{fb}_\bullet(\wt;\x) = \langle\wt^{fb},\x\rangle\cdot (\langle \wt^{ff},\x\rangle-\vartheta)\cdot \bbO_j - A_\bullet(\wt)$ 
obtains
\begin{equation}
	\Delta \wt_{ij}^{ff} \propto \langle\wt^{fb}_{j},\x\rangle\cdot \bbO_{ij} - \partial_{i} A_\bullet(\wt),
\end{equation}
which differs from AGREL \eqref{e:agrel} by using $\partial_{i}A_\bullet$ as regularizer instead of $(1-\x_j)$ and extending feedback from a single neuron, $\wt^{fb}_{kj}\x_k$, to many neurons, $\langle\wt^{fb}_{j},\x\rangle$. We also drop the global reward signal $f(\delta)$ since we are interested in the pure backpropagation case; it can easily be reinstated.

Note that the utility function $\mu^{fb}$ is itself plastic. Neuron $n_j$ not only modifies feedforward weights to maximize its score, it also modifies feedback weights to increase the maximum achievable score:
\begin{equation}
	\Delta \wt_{kj}^{fb} \propto (\langle\wt^{ff}_j,\x\rangle-\vartheta)\cdot\bbO_{jk} - \partial_{k} A_\bullet(\wt).
\end{equation}

\subsection{Estimating usefulness with feedback}
As suggested in the introduction, one way to encourage collaboration is for each neuron to estimate its usefulness to the rest of the system and optimize that estimate. By Corollary~\ref{t:stdp_proper}, a faithful measure of the usefulness of $n_j$'s output to the rest of cortex is the sum of active downstream synaptic weights:

\begin{defn}
	The \textbf{usefulness} $\V_j(\x)$ of a spike by $n_j$ is the sum of the synaptic weights of downstream neurons that co-spike with $n_j$: 
	\begin{equation}
		\label{e:usefulness}
		\V_j(\x):=\sum_{\{k \,|\, j\rightarrow k\}} \wt_{jk}\bbO_{jk} = \langle\wt_{j\bullet},\x\rangle\bbO_j.	
	\end{equation}
	Intuitively, $\V_j(\x)$ is the total utility that spiking downstream neurons expect after $n_j$ spikes.
\end{defn}

Neurons cannot compute their usefulness directly, since the utilities of downstream neurons are private. They must therefore make do with publicly available data: \emph{spikes by other neurons}. We therefore propose that neurons use feedback, which they can actually compute, as a proxy for usefulness, which would be ideal. 

As a consequence of Corollary~\ref{t:stdp_proper}, we quantify how closely feedback-utility approximates usefulness \eqref{e:usefulness}:

\begin{cor}[estimating usefulness with feedback]
	\label{t:feedback}
	Neuron $n_j$ equipped with utility function $\mu^{fb}(\x)$ approximately maximizes its usefulness $\V(n_j)$ to the rest of cortex, where the failure of the approximation is
	\begin{equation*}
		\sum_{k} \bbO_{jk}\left[
		\overbrace{\underbrace{G_\bullet\Big(\bE\big[\mu_k(\x) \bbO_{jk}\big]\Big)}_{\wt_{jk}}}^{\text{usefulness}}
		- \overbrace{\underbrace{G_\bullet\Big(\bE\big[\langle \wt^{ff}_j,\x\rangle\bbO_{jk}\big]\Big)}_{\wt^{fb}_{kj}}}^{\text{approximation}}
		\right].
	\end{equation*}
\end{cor}
Thus, the quality of $\langle\wt^{fb}_{j},\x\rangle$ as an estimate of $\V_j(\x)$ depends on how closely $n_j$'s feedforward inputs $\langle\wt^{ff}_j,\x\rangle$ approximate the sum of the downstream utilities $\mu_k(\x)$. 	

\begin{proof}
	The usefulness and utility of $n_j$ are 
	\begin{equation*}
		\V_j(\x)  = \langle\wt_{j\bullet},\x\rangle\bbO_j 
		\quad\text{and}\quad
		\mu^{fb}_j(\x)  = \langle\wt^{fb}_{\bullet j},\x\rangle
		\quad\text{respectively.}
	\end{equation*}
	The utility $\mu^{fb}_j$ is multiplied by $\bbO_j$ when it is used in scoring rules, so the difference comes down to the weights.
	Corollary~\ref{t:stdp_proper} implies the optimal feedforward weights are
	\begin{equation*}
		\wt_{jk} = G_\bullet\Big(\bE_{P}\big[\mu_k(\x)\cdot \bbO_{jk}\big]\Big)
	\end{equation*}
	so that the usefulness of $n_j$ is
	\begin{equation*}
		\V_j(\x) = \sum_{\{k \,|\, j\rightarrow k\}} \bbO_{jk} \cdot G_\bullet\Big(\bE_{P}\big[\mu_k(\x)\bbO_{jk}\big]\Big).
	\end{equation*}
	Again by Corollary~\ref{t:stdp_proper}, the properness of the scoring function implies the optimal weights for $k\rightarrow j$ satisfy
	\begin{equation*}
		\wt^{fb}_{kj} 
		= G_\bullet\Big(\bE_{P}\big[\langle \wt^{ff},\x\rangle \bbO_{jk}\big]\Big)
	\end{equation*}
	and we are done.
\end{proof}

Experiments in \S\ref{s:experiments} demonstrate that $\wt^{fb}_{kj}$ is a good proxy for $\wt_{jk}$ in some interesting cases. 

A detailed analysis of the relationship between approximations $\wt_{jk}\approx \wt^{fb}_{kj}$, distribution $P(\x)$, and utility functions $\mu_k(\x)$ is beyond the scope of this paper.

\subsection{Neurons as rational agents}
\label{s:cpm}


Section~\S\ref{s:npm} suggested neurons are analogous to markets in which synapses trade. This section presents a second analogy, where cortex forms a market in which neurons trade.

Recall that neurons are rational agents that optimize their expected reward balanced against a cost term, Theorem~\ref{t:limit}:
\begin{equation*}
	\wt_j^*:= \argmax_{\wt_j\in\cHh} \,\bE_P\Big[\big(\langle\wt_{j},\x\rangle-\vartheta\big)\cdot \mu_j(\x)\bbO_j - A_\bullet(\wt_{j})\Big].
\end{equation*}

The key idea is that each neuron should optimize its usefulness to the rest of the brain. Building on Corollary~\ref{t:stdp_proper}, usefulness is defined as $\V_j(\x)=\sum_{j\rightarrow k} \wt_{jk}\bbO_{jk}$. That is, the quantity of $\bbO_j$ used by downstream neurons in their internal markets. Unfortunately, $n_j$ does not have access to this number. Similarly to how musicians are paid for actual sales rather than downloads of their music, neurons need to record when their outputs are used. They therefore use feedback to compute $\langle\wt^{fb}_{\bullet j},\x\rangle\bbO_j$, which acts as a proxy\footnote{Recorded usage could over- or under- estimate true usage. Section~\S\ref{s:experiments} shows that it is a good guide in practice.} for $\langle\wt^{fb}_{\bullet j},\x\rangle\bbO_j$.

\begin{center}
\begin{tabular*}{.6\textwidth}{@{\extracolsep{\fill}}| l | l | l |}
	\hline
	$n_j$ & neuron & trader\\
	$\langle\wt^{ff}_{\bullet j},\x\rangle\bbO_j$ & ff current $\times$ spike & purchases by $n_j$ \\
	$\langle\wt_{j\bullet},\x\rangle\bbO_j$ & usefulness $\V_j(\x)$ of $n_j$ & use made of $n_j$\\
	$\langle\wt^{fb}_{\bullet j},\x\rangle\bbO_j$ & fb current $\times$ spike & recorded usage\\
	&  & = payment to $n_j$\\
	\hline
\end{tabular*}
\end{center}

Intuitively, $n_j$ simultaneously sets its feedback connections on the downstream traders that most frequently purchase its spikes, and sets its feedforward connections on the upstream traders that sell the most useful spikes. 

The result is a mesh of intertwining neuronal chains -- optimized for usefulness at every link by the invisible hand of the cortical market -- that connects sensory inputs to motor actions.


\section{Experiments}
\label{s:experiments}

\begin{figure}
	\begin{center}
	    \includegraphics[width=4in]{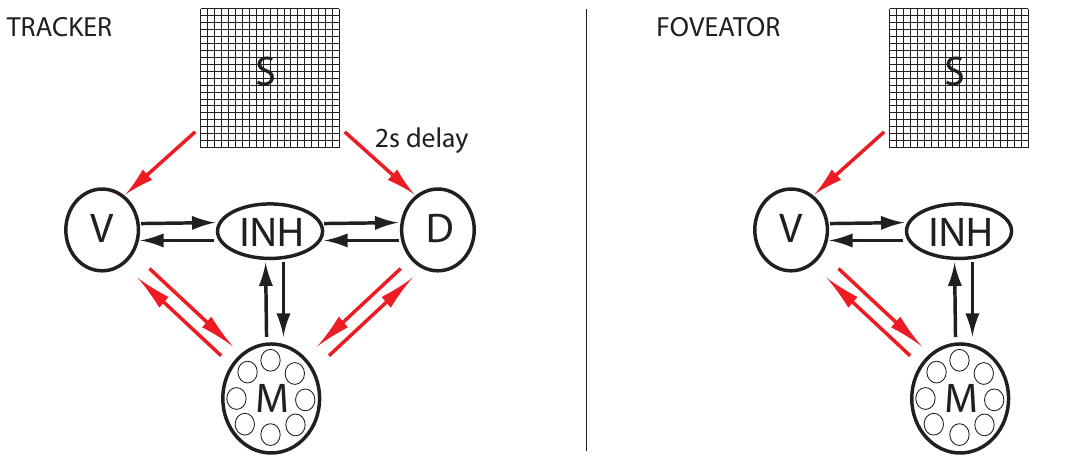}		
	\end{center}
  \caption{Foveator and Tracker architectures. Arrows are initialized randomly. Red arrows are plastic; black are fixed.
  }
  \label{f:arch} 
\end{figure}

We investigate the empirical performance of discretized neurons. The experiments are designed to show that:  (i) the ideas above can be implemented with minimal modifications; (ii) synaptic weights encode environmental statistics and rewards; (iii) feedback improves performance; and (iv) feedback reliably estimates a neuron's usefulness. 

We have therefore constructed networks, inspired by \citep{Nere:2012fk}, that learn tasks designed so that the embedding of expected utilities into synaptic weights is easy to visualize. 

Our goal is not to compete with the state of the art. Rather, our aim is to introduce mechanism design techniques into the analysis and construction of networks.  A pressing open question is whether more sophisticated networks, such as those developed by the deep learning community, can be understood or improved via mechanism design.

\paragraph{Network architectures} 
The \textbf{tracker network}, Fig.~\ref{f:arch} left, has a sensory grid $S$ of $20\times20$ neurons, intermediate layers $V$ and $D$ with 100 neurons each, motor layer $M$, and 100 randomly connected inhibitory neurons $INH$. Signals from $S$ to $D$ are \emph{delayed} so that $V$ and $D$ receive different temporal snapshots of $S$. Synapses are plastic except those to or from inhibitory neurons. $M$ is divided into 8 areas of 10 neurons each. Actuators engage when they receive more than 10 spikes. The network is initialized randomly. 

The tracker network tracks targets traveling along an edge of the visual field. Motor areas are rewarded ($\mu_j=+1$) or punished ($\mu_j=-1$) according to whether or not the action correctly anticipates where the target is headed and from which direction ($4\times 2$ possibilities).
Note the motor layer receives neuromodulatory signals whereas the intermediary layers do not and learn from feedback. 

The \textbf{foveator network}, Fig.~\ref{f:arch} right, drops $D$ and has fewer inhibitory neurons. 

The foveation task is to move the fovea (center of the retina) onto an object appearing on the edge of the visual field. Each motor area controls an actuator that moves the fovea in a compass direction (N, NW, W, etc). After a movement, the corresponding area is rewarded ($\mu_j=1$) if the object is closer to the center and punished ($-1$) otherwise.

\begin{figure*}[t]
	\centering	
	\subfigure[Synapses on paths through subsystem $V$]{\includegraphics[width=2.5in]{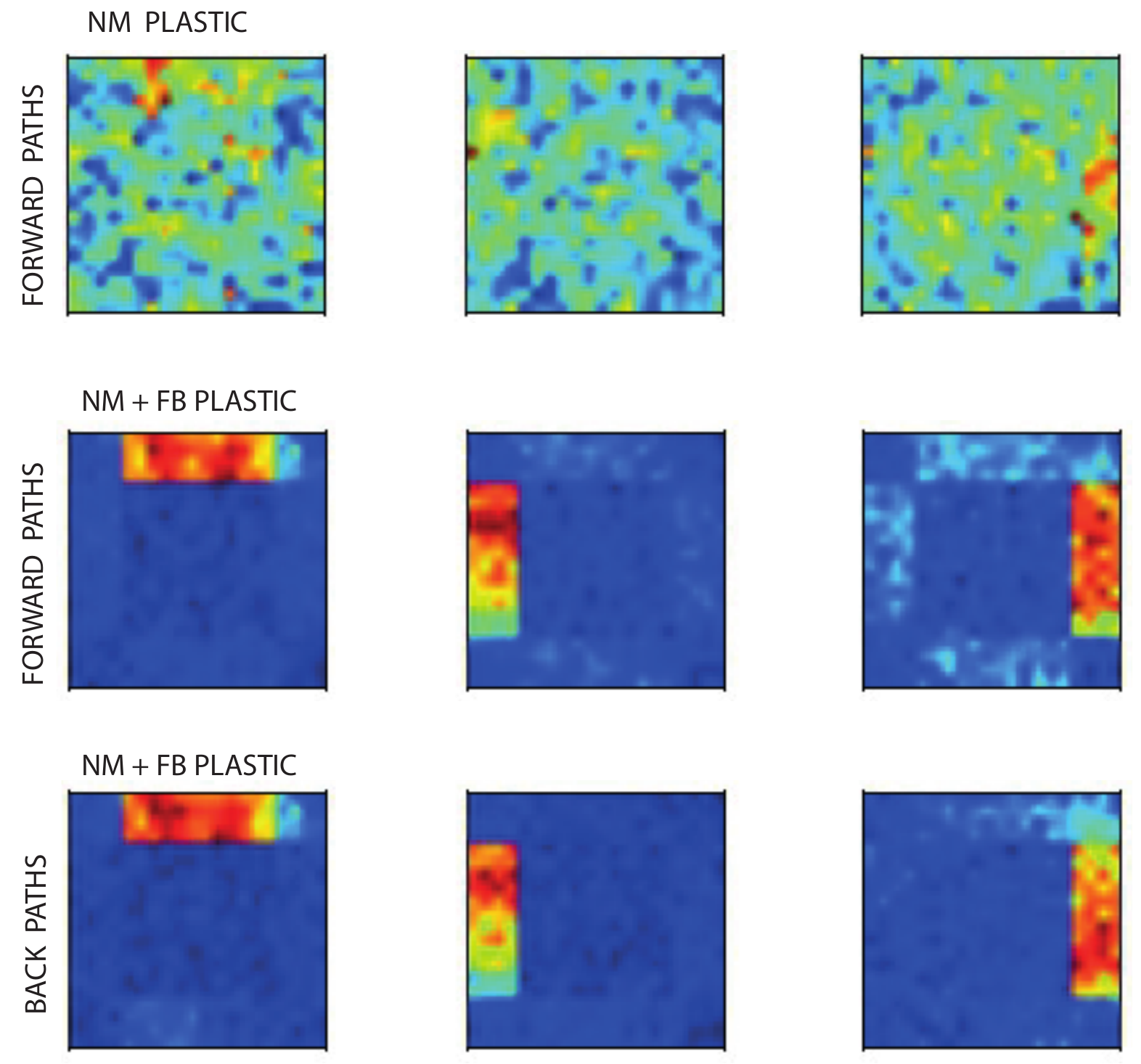}
	\label{f:trk-v}}
	\hspace{10mm}
	\subfigure[Synapses on paths through subsystem $D$]{\includegraphics[width=2.5in]{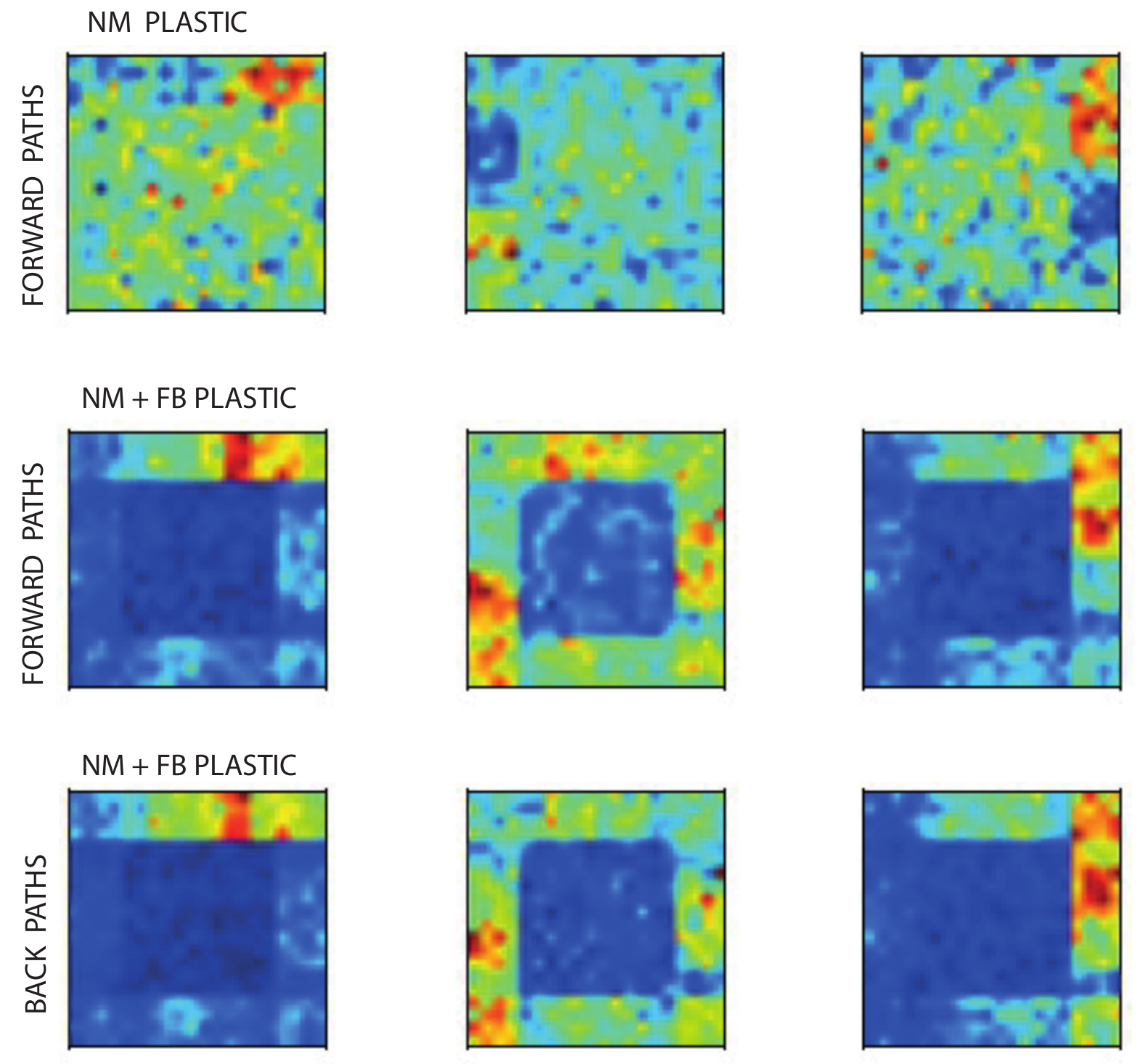}
	\label{f:trk-d}}
	\vspace{-3mm}
 	\caption{Average synaptic weights in Tracker network.
	}
	\label{f:trk}
\end{figure*}

We tweak the discretized neuron in \S\ref{s:minimal} to make the dynamics closer to continuous time models of cortical neurons.

First, we introduce a voltage term $V$, which provides neurons with a steadily decaying ``memory'' of previously received spikes. Neurons spike when $V>\vartheta$, after which $V\leftarrow0$. When the neuron does not spike, $V$ is updated according to
\begin{equation*}
	V\leftarrow V + \langle\wt,\x\rangle - \delta.
\end{equation*}
Neurons maintain an exponentially decaying trace reflecting recent output spikes:
\begin{equation*}
	\mathtt{trace}_j \leftarrow 0.95\cdot\big(\mathtt{trace}_j + 0.4\cdot\bbO_j\big)
\end{equation*}
Neurons in subsystems $V$ and $D$ update their feedforward synapses according to
\begin{equation*}
	\Delta\wt_{ij} \propto  \langle\wt^{fb},\x\rangle\cdot\bbO_i\cdot\mathtt{trace}_j
\end{equation*}
and similarly for feedback. Thus, $\mathtt{trace_j}$ is substituted for $\bbO_j$ to temporally smooth out learning. Neurons in subsystem $M$ update their synapses according to
\begin{equation*}
	\Delta\wt_{ij} \propto \mu_j(\x)\cdot \sum_{t'=t-m}^{t} \bbO_{ij}
\end{equation*}
where the sum is over tics since the last neuromodulatory signal, similar to the trace implemented in $V$ and $D$. 

Finally, we tweak the regularization. Instead of continually regularizing by $A_1(\wt)$, we regularize at discrete intervals, analogous to a hypothesized role of sleep \citep{bb:12}. Regularization consists of setting the $K$ strongest synapses to 1, and pruning the rest (i.e. setting their weights to 0). The number $K$ is fixed within each layer, but varies across layers.

\paragraph{Visualizing synaptic strengths}
By Corollary~\ref{t:feedback}, the quality of a neuron's usefulness-estimate can be computed and visualized by comparing average feedforward and feedback weights. 

Weights are visualized in Fig \ref{f:trk} and \ref{f:fov} as follows. For each area in $M$, we average over all \emph{feedforward} paths $S\rightarrow V\rightarrow M$ and \emph{feedback} paths $S\rightarrow V\leftarrow M$ respectively, and similarly for $D$. To save space, 3 out of 8 areas in $M$ are plotted. Plots are averaged over 20 runs. Blue denotes low values; red denotes high.

\paragraph{Results}
(i) The tasks are easy and the networks rapidly (within a few thousand tics) achieve 98\% and 95\% accuracy. 
The tracker outperforms the foveator, possibly because the foveator modifies its environment by actively moving the center of the retina, whereas the tracker does not.

The delay line is essential to tracker performance: if the delay is set to zero then the network performs little better than chance, and the structure of the environment is not learned at all.

\vspace{1mm}
\begin{center}
\begin{tabular*}{.6\columnwidth}{@{\extracolsep{\fill}}| l | c c |}
	\hline
	\emph{Tracker} & \% correct & \# correct\\
	\hline
	only $M$ plastic & 95 & 243 \\
	all plastic & 98 & 672 \\
	\hline
	\emph{Foveator} & & \\
	\hline
	only $M$ plastic & 93 & 80 \\
	all plastic & 95 & 207 \\
	\hline
\end{tabular*}
\end{center}
\vspace{1mm}

(ii) The middle rows of Fig~\ref{f:trk} and \ref{f:fov} show how rewards and environmental statistics are incorporated into the networks' feedforward structure. 

For the tracker network, the $V$-area synapses learn trajectories; whereas the $D$-area learns the \emph{starting points} of trajectories. The combination of instantaneous lines in $V$ (which learn directions) and delay lines in $D$ (which learn starting points) thus allows the network to implicitly compute derivatives and thereby determine directions of travel. 

For the foveator, it is easy to read off the correspondence between the NE, N, and NW movements of the actuators and the locations of objects driving the movements.

(iii) Shutting off feedback plasticity (top rows of Fig~\ref{f:trk} and \ref{f:fov}) slightly worsens performance, from 98\% to 95\% for the tracker and from 95\% to 93\% for the foveator. However, it dramatically worsens the ``reaction times'' of the networks, quantified as the number of times the actuators correctly engage per 1000 tics. 

Indeed, looking at synaptic weights \emph{without} feedback plasticity, top rows of the figures, we find that the structure of the rewards and environment is barely visible.

(iv)  Finally, when feedback plasticity is turned on, average synaptic weights over feedforward paths $T\rightarrow V/D\rightarrow M$ (middle rows) and feedback paths $T\rightarrow V/D\leftarrow M$ (bottom rows) are almost identical, demonstrating that neurons in $V$ and $D$ accurately estimate their usefulness to downstream $M$ neurons using feedback.

\section{Conclusion}
\label{s:discussion}

This paper applied tools from mechanism design to investigate a simple model of cortical neurons. The main result is that, under a technical assumption, neurons faithfully encode expected utilities into their synaptic weights. If the result can be extended to more realistic models, then it will provide a powerful new approach to understanding the relationship between cortical structure and function.

There is good reason to be optimistic: extending the analysis to continuous time models requires exponential discount factors analogous to interest rates -- which are well-understood in mechanism design. 

An important corollary of the analysis is a novel interpretation of the role of spiking feedback in cortex: neurons can use feedback spikes to estimate their usefulness to the rest of cortex, and then learn to maximize that estimate. 

We have used the the simplest possible scoring rules, derived from standard models of neurons, to provide a proof of principle. It will be interesting to explore more realistic models taken from the neuroscience literature, and also more powerful models such as those developed for deep neural networks. 

Finally, although the flow of ideas in this paper is one-sided -- from mechanism design to neuronal models -- we expect that future work will be more symmetric. The cortex aggregates information far more effectively than the auctions and online markets studied in game theory. This suggests there are powerful design principles waiting to be uncovered. 

\vspace{3mm}
\noindent
\textbf{Acknowledgements.}
I am grateful to Hastagiri Vanchinathan for encouraging me to look into mechanism design.


\begin{figure}
	\begin{center}
	    \includegraphics[width=2.5in]{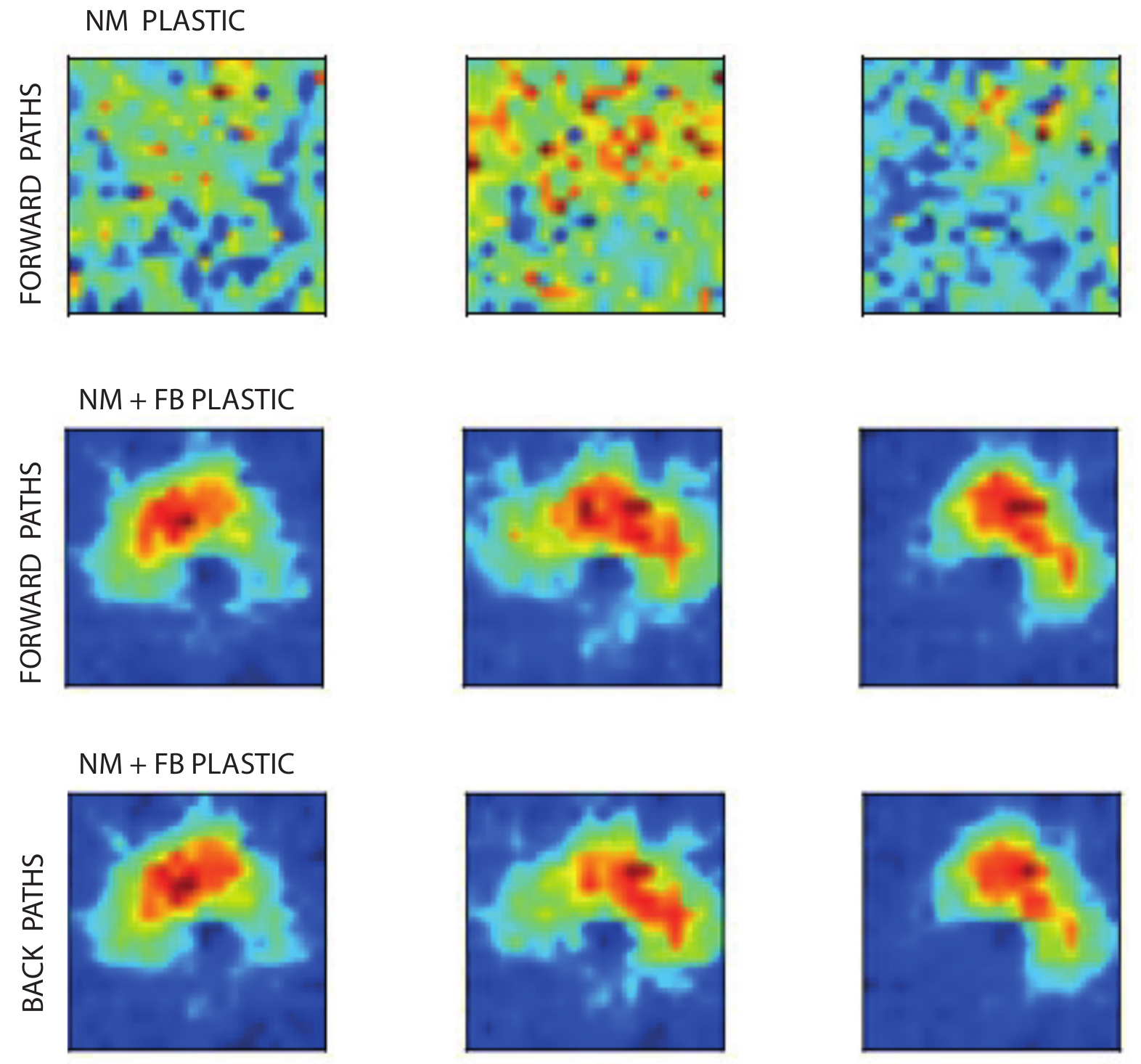}		
	\end{center}
	\vspace{-3mm}
  \caption{Average weights in Foveator network.
  }
  \label{f:fov} 
\end{figure}

{

}

\end{document}